\documentclass[letterpaper]{article} 
\usepackage{aaai2027}
\nocopyright
\usepackage[hyphens]{url}  
\usepackage{graphicx} 
\usepackage{amsmath} 
\usepackage{amssymb} 
\usepackage{amsthm} 
\newtheorem{theorem}{Theorem}
\newtheorem{proposition}{Proposition}

\newtheorem{lemma}{Lemma}
\usepackage{tikz}
\usetikzlibrary{arrows.meta,positioning,fit,backgrounds,shapes.geometric}
\usepackage{multirow}
\usepackage{natbib}  
\usepackage{caption} 
\usepackage{algorithm}
\usepackage{algorithmic}

\usepackage{newfloat}
\usepackage{listings}
\DeclareCaptionStyle{ruled}{labelfont=normalfont,labelsep=colon,strut=off} 
\floatstyle{ruled}
\newfloat{listing}{tb}{lst}{}
\floatname{listing}{Listing}

\usepackage{booktabs}

\title{Look Again Before You Abstain:\\Budgeted Conformal Evidence Acquisition for Reliable Vision-Language Models}
\author{
    Jian Xu\textsuperscript{\rm 1,2},
    Yanning Wu\textsuperscript{\rm 3},
    Delu Zeng\textsuperscript{\rm 3},
    John Paisley\textsuperscript{\rm 4},
    Qibin Zhao\textsuperscript{\rm 2}
}
\affiliations{
    \textsuperscript{\rm 1}RIKEN iTHEMS\quad
    \textsuperscript{\rm 2}RIKEN AIP\quad
    \textsuperscript{\rm 3}South China University of Technology\quad
    \textsuperscript{\rm 4}Columbia University\\
    jian.xu@riken.jp
}

\begin{document}
\maketitle

\begin{abstract}
Large vision-language models (LVLMs) hallucinate, asserting visual details the image does not support. Selective prediction with a distribution-free guarantee---abstain unless a claim is grounded---bounds the hallucination rate among asserted claims, but at a brutal price: to hold it below $5\%$ on balanced object-existence claims, a conformal grounding filter must abstain on over $80\%$. We argue abstention is wasteful when more visual evidence is cheap, and introduce \textbf{Budgeted Conformal Evidence Acquisition (BCEA)}, which turns the binary answer/abstain decision into a three-way choice: answer, abstain, or \emph{acquire} more evidence by re-examining the image (zoom, crop, or a claim-specific intervention) under a compute budget. Our key observation: acquisition plugged naively into a calibrated filter \emph{breaks the guarantee}---it destroys the exchangeability conformal calibration relies on---but folding the whole acquisition policy into the score and re-calibrating \emph{restores} the finite-sample guarantee while recovering coverage. Where a global score fails, a claim-specific intervention can supply signal: a horizontal flip lifts spatial-relation verification from chance ($0.57$ to $0.77$ AUROC). The exact variant is certified (Bonferroni--Clopper--Pearson; zero empirical violations over $300$ splits) and, gated, beats abstention at moderate risk; it also transfers to claims a VLM \emph{generates itself}---on free-form captions ($20.6\%$ hallucinated) it certifies $34\%$ of claims at a guaranteed $\le10\%$ rate vs.\ $24\%$ for abstention. The exact certificate holds across all four open VLMs (LLaVA-1.5, Qwen3-VL, LLaVA-NeXT, InternVL2); its \emph{practical}, uncertified variant additionally improves coverage empirically on POPE while keeping risk near the target.
\end{abstract}

\section{Introduction}

Large vision-language models (LVLMs) produce fluent, helpful descriptions of images, but they also \emph{hallucinate}: they assert objects, attributes, and relations that are not present in the visual input \citep{leng2024mitigating, huang2024opera, guan2024hallusionbench}. For deployment in settings that demand reliability---medical, scientific, assistive---fluency is not enough; what is needed is a \emph{guarantee} that, whatever the model says, the rate of unsupported claims stays below a user-specified level.

Conformal prediction offers exactly this kind of guarantee, and a recent line of work brings it to LVLMs. The strongest instance, \textsc{ConfLVLM} \citep{li2025towards}, treats each generated detail as a hypothesis, scores it with a cheap uncertainty heuristic, and filters out claims whose score falls below a conformally calibrated threshold; the result is a finite-sample, distribution-free bound on the factual error rate of what remains. This is principled and attractive. But it inherits the central weakness of selective prediction: \emph{to be safe, it abstains}. In our experiments, guaranteeing a hallucination rate below $5\%$ forces such a filter to abstain on more than $80\%$ of all claims (Figure~\ref{fig:motivation}). A reliable model that refuses four out of five questions is of limited use.

We start from a simple observation: \emph{abstention is wasteful when more evidence is cheaply available.} When a person is unsure whether a small object is present, they look closer; they do not give up. LVLMs admit the analogous move---re-encoding a zoomed crop, or re-scoring the claim under a targeted visual intervention---at the cost of a few extra forward passes. The question is whether this ``look again'' step can be added \emph{without forfeiting the statistical guarantee} that motivated selective prediction in the first place.

Existing work addresses the two ingredients separately. Conformal generation methods such as \textsc{ConfLVLM} and CAP \citep{li2025towards,tayebati2025learning} provide guarantees for claim filtering or abstention, while CEBC \citep{mishra2026cebc} uses a conformally calibrated detector-evidence threshold to edit or suppress unsupported object mentions; these methods do not adaptively crop, magnify, and re-encode an additional image view. Conversely, \textsc{ReCoVERR} \citep{srinivasan2024selective} elicits additional image-conditioned evidence through LLM-generated follow-up questions, while recent visual re-examination methods magnify or crop task-relevant regions \citep{mao2026magnifying,yi2026beyond}. These methods report empirical improvements rather than a finite-sample, distribution-free guarantee on post-acquisition selective risk. To our knowledge, prior work has not combined adaptive acquisition and re-encoding of additional image views with such a guarantee under a bounded acquisition budget; as we show, naively applying a pre-acquisition threshold to post-acquisition scores is not safe.

We introduce \textbf{Budgeted Conformal Evidence Acquisition (BCEA)}. BCEA replaces the binary answer/abstain decision with a three-way action---answer, abstain, or \emph{acquire}---and makes three contributions.

\paragraph{(1) Acquisition that is naive breaks the guarantee.} Calibrating a conformal threshold on a claim's base score and then, at test time, replacing that score with a higher post-acquisition score is the obvious thing to do, and it is wrong. Acquisition changes the score distribution, breaking the exchangeability between calibration and test on which conformal validity rests. Empirically, naive acquisition raises coverage but violates the risk guarantee on almost every split (Table~\ref{tab:acq}). We make this failure precise.

\paragraph{(2) Re-calibration on post-acquisition scores restores it.} The fix is to fold the entire acquisition policy---crops, interventions, and the gate---into the score and calibrate on the post-acquisition scores, so calibration and test pass through one \emph{sample-independent} map. The principle (folding a fixed, label-independent transformation into a conformal pipeline) is standard; our contribution is to pin down what it requires here---the gate must not depend on the calibration sample, not merely its labels---and to show the resulting exact procedure is theoretically certified while still improving guaranteed coverage over abstention at moderate risk.

\paragraph{(3) Structured interventions beat global ungrounding.} The grounding score of prior conformal filters---how much the claim's likelihood drops when the image is removed---works for object existence but \emph{fails} for relational claims: removing the image lowers ``\emph{A} is left of \emph{B}'' and ``right of'' equally. Cheap claim-type-specific interventions recover the lost signal---a single horizontal flip lifts spatial-relation verification from $0.57$ to $0.77$ AUROC. We further characterize \emph{when} acquisition helps: the optimal coverage at a fixed risk is a point on the score's ROC curve, so lifting the ROC raises it (Theorem~\ref{thm:roc}), which explains why model-guided crops beat a uniform grid while per-claim budget allocation does not. We validate all of this on COCO and POPE existence/relation claims across four open VLMs, with head-to-head baselines.

\begin{figure}[t]
\centering
\includegraphics[width=0.92\columnwidth]{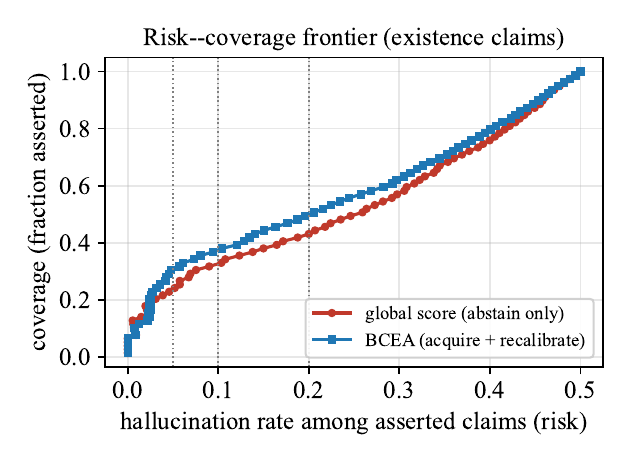}
\caption{Abstention is conservative. \emph{Empirical} risk--coverage frontier on $1{,}440$ balanced object-existence claims (LLaVA-1.5-7B; each point is one threshold, not a certified curve). The abstain-only global score gives up $82\%$ of claims to reach $5\%$ risk; acquisition (BCEA) shifts the frontier up. Finite-sample \emph{certified} coverage (lower) is in Table~\ref{tab:acq}.}
\label{fig:motivation}
\end{figure}

\section{Related Work}

\paragraph{Hallucination in LVLMs.} Object hallucination is measured by POPE \citep{li2023evaluating}, AMBER \citep{wang2023amber}, THRONE \citep{kaul2024throne}, HallusionBench \citep{guan2024hallusionbench}, and TUBench \citep{he2024tubench}. Mitigations include contrastive or multi-scale decoding (VCD \citep{leng2024mitigating}, HALC \citep{chen2024halc}, SECOND \citep{park2025second}), over-trust penalties (OPERA \citep{huang2024opera}), and preference optimization (RLHF-V \citep{yu2024rlhfv}, V-DPO \citep{xie2024vdpo})---improving empirical metrics without a finite-sample, distribution-free guarantee.

\paragraph{Conformal and selective prediction for generation.} Selective prediction trades coverage for controlled risk \citep{geifman2017selective}. Conformal factuality \citep{mohri2024language} and conformal abstention \citep{yadkori2024mitigating} bound LM error by backing off or abstaining. For LVLMs, \textsc{ConfLVLM} \citep{li2025towards} filters unreliable claims with a distribution-free factuality guarantee, CAP \citep{tayebati2025learning} learns per-instance conformal risk levels, and CEBC \citep{mishra2026cebc} edits unsupported mentions via a calibrated detector threshold. None adaptively acquires and re-encodes an additional visual view.

\paragraph{Evidence acquisition and active perception.} \textsc{ReCoVERR} \citep{srinivasan2024selective} elicits LLM follow-up questions as textual evidence under a heuristic tolerance, without a distribution-free guarantee. Zoom/crop and visual-search methods---ViCrop \citep{zhang2023towards}, CropVLM \citep{carvalho2026cropvlm}, V* \citep{wu2024vstar}, AdaptVision \citep{lin2026adaptvision}, Perception Magnifier \citep{mao2026magnifying}, HiPerson \citep{yi2026beyond}---target accuracy or efficiency, not guaranteed risk. BCEA instead treats visual re-examination as a \emph{calibrated verification action} and calibrates the resulting acquire-and-verify pipeline for finite-sample selective-risk control.

\paragraph{Conformal validity under adaptivity.} Folding a fixed, label-independent transformation into a conformal pipeline is known to preserve exchangeability---for test-time augmentation \citep{shanmugam2025test}, compute-budgeted LLM reasoning \citep{wang2026conformal}, joint calibration--test thresholds \citep{xu2025selective}, and monotone-loss risk control \citep{angelopoulos2024conformal}. We instantiate this principle, but contribute (i) that \emph{visual} acquisition silently breaks the guarantee unless the policy is folded into the score, and (ii) acquisition in the pixels via claim-type interventions. To our knowledge BCEA is the first to couple guaranteed VLM hallucination filtering with visual evidence acquisition.

\section{Preliminaries}

\paragraph{Claim-level selective prediction.} A response is decomposed into atomic claims $c_1,\dots,c_m$ \citep{li2025towards}. Each claim $c$ about image $x$ has an unknown binary label $y\in\{0,1\}$ ($1$ = supported by the image). A scoring function $s(x,c)\in\mathbb{R}$ ranks claims by how grounded they appear (higher = more grounded). A selective rule \emph{asserts} $c$ iff $s(x,c)\ge\tau$ and otherwise abstains. We measure
\begin{align}
\text{coverage} &= \Pr[s\ge\tau], &
\text{risk} &= \Pr[y=0 \mid s\ge\tau],
\end{align}
i.e.\ the fraction of claims asserted and the hallucination rate among asserted claims.

\paragraph{Split-conformal selective risk control.} Given a calibration set of i.i.d.\ (or exchangeable) labeled claims and a target level $\alpha$, we choose the smallest threshold $\tau$ (maximizing coverage) such that an upper confidence bound on the calibration risk is at most $\alpha$. For accepted calibration claims with $k$ errors out of $n$, we use the exact Clopper--Pearson upper bound $\overline{R}_\delta(k,n)=\mathrm{Beta}^{-1}_{1-\delta}(k+1,\,n-k)$ (with $\overline{R}_\delta{=}1$ when $n{=}0$ or $k{=}n$). Over a \emph{deterministic} grid $\Lambda$ of thresholds (fixed independently of the calibration sample), the exact procedure selects $\hat\tau=\min\{t\in\Lambda:\overline{R}_{\delta/|\Lambda|}(k(t),n(t))\le\alpha\}$, abstaining if none; the Bonferroni level $\delta/|\Lambda|$ makes this control the selective risk at level $\alpha$ with probability $1-\delta$ on exchangeable test claims. Scanning the calibration scores at level $\delta$ \emph{without} this correction gives the higher-coverage but \emph{uncertified} practical variant.

\section{Method: BCEA}

\begin{figure*}[t]
\centering
\vspace{-5mm}
\includegraphics[width=\textwidth]{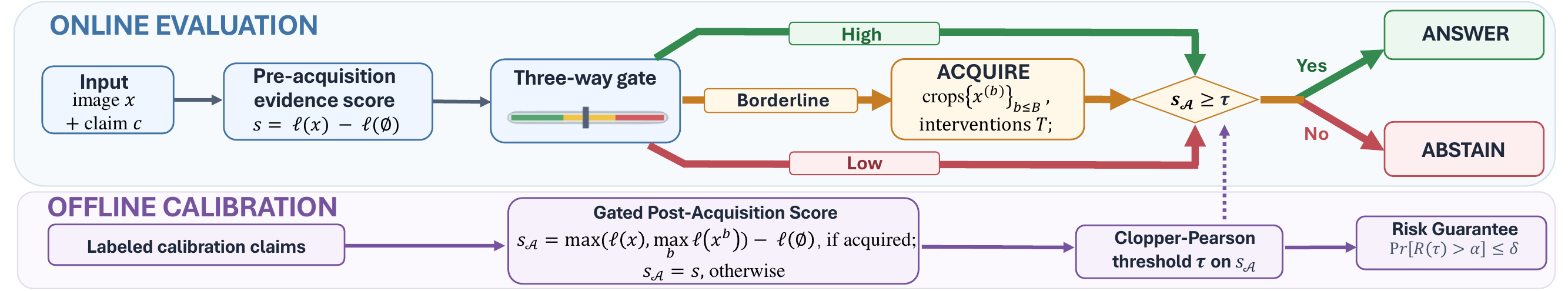}
\caption{BCEA replaces answer/abstain with a three-way decision. A claim is scored by evidence sufficiency; borderline claims trigger \emph{acquisition} (zoom crops / claim-type interventions), producing a post-acquisition score $s_{\mathcal A}$. The conformal threshold $\tau$ is calibrated \emph{on $s_{\mathcal A}$ itself}, so that asserted claims carry a distribution-free hallucination-rate guarantee $\le\alpha$ (Theorem~\ref{thm:valid}). Calibrating $\tau$ on the pre-acquisition score instead silently breaks the guarantee (Proposition~\ref{prop:naive}).}
\label{fig:arch}
\end{figure*}

BCEA turns a binary answer/abstain filter into a guaranteed, evidence-seeking one (Figure~\ref{fig:arch}). We first define the scores, then the acquisition action, the validity argument, and a characterization of when acquisition helps.

\subsection{Evidence-Sufficiency Scores}

\paragraph{Global ungrounding.} Following the intuition behind prior conformal filters, a claim is grounded if its likelihood depends on the image. With $\ell(v)=\frac{1}{|c|}\sum_t \log p_\theta(c_t\mid v, \text{prompt})$ the mean token log-likelihood of claim $c$ under visual input $v$, the \emph{global} score is
\begin{equation}
s_{\text{glob}}(x,c) = \ell(x) - \ell(\varnothing),
\end{equation}
where $\varnothing$ is a blank (all-black) image. A hallucinated claim barely changes when the image is removed ($s_{\text{glob}}\approx 0$); a grounded claim drops sharply.

\paragraph{Region and intervention scores.} $s_{\text{glob}}$ throws the whole image away and cannot tell \emph{where} the evidence is, nor distinguish claims that depend on visual \emph{structure}. We replace the blank reference with a targeted intervention $T$ chosen by claim type, $s_T(x,c)=\ell(x)-\ell(T(x))$: region masking ($T$ blacks out the queried object's support) for existence/attribute claims, and a horizontal flip ($T(x)=\text{flip}(x)$) for left/right relational claims, under which a \emph{correct} directional claim becomes false (so $s_T>0$) while an \emph{incorrect} one becomes true (so $s_T<0$).

Figure~\ref{fig:qual} illustrates both score families on real images with measured likelihoods.

\subsection{The Acquisition Action}

When $s(x,c)$ is too low to assert but not clearly ungrounded, BCEA \emph{acquires}: it forms a set of $B$ zoomed views $\{x^{(1)},\dots,x^{(B)}\}$ of $x$ (a coarse grid of overlapping crops, requiring no ground-truth localization) and forms the post-acquisition score
\begin{equation}
s_{\text{acq}}(x,c) = \max\!\big(\ell(x),\ \max_{b\le B}\ell(x^{(b)})\big) - \ell(\varnothing).
\label{eq:sacq}
\end{equation}
For a true object, some crop magnifies it and raises its likelihood; for a hallucinated object, no crop provides support. The maximization is the source of both the benefit and the danger analyzed next.

The three-way answer/abstain/\emph{acquire} decision is realized by \emph{gating}: run the $B$ crops only when the base grounding $s$ falls inside a borderline band, and answer/abstain directly otherwise. Validity here needs care, and label-free is \emph{not} enough: were the band a \emph{quantile} of the calibration scores, each calibration claim's policy would depend on the calibration sample containing it, and calibration and test would no longer pass through one fixed map---Bonferroni over $(\text{band},\tau)$ would not repair that. We therefore define bands by \emph{pre-fixed numeric} bounds on the base score (e.g.\ $[-1,1]$), taken from the same a-priori grid and never from a data quantile, so $s_{\mathcal A}$ is sample-independent and Theorem~\ref{thm:valid} applies verbatim. Figure~\ref{fig:gate} plots the compute--coverage trade-off across pre-fixed bands: a narrow band ($\approx2$ extra passes) reaches guaranteed coverage $\approx0.26$, already above full budget ($0.19$), because acquiring on every claim inflates the score-max for false claims too, which the conservative certificate penalizes. \emph{Adaptively} selecting band \emph{and} threshold per split under a single Bonferroni correction over all (band,$\tau$) pairs---which keeps the guarantee valid---reaches $0.32$ at $\alpha{=}0.10$, above No-Acq ($0.23$); this is the gated-selected row of Table~\ref{tab:acq}. Choosing the band on an independent pilot split instead gives the same conclusion ($0.23/0.41$ at $\alpha{=}0.10/0.20$, violation $0$).

\subsection{Why Naive Acquisition Is Invalid---and How to Fix It}

The results below are elementary consequences of one principle---a fixed, label-independent transformation may be folded into a conformal pipeline without breaking coverage---which we use to pin down \emph{which} step of visual acquisition violates it; our contribution is empirical, and the takeaway is that the gate must be calibration-sample-independent, not merely label-free. Concretely, let $\tau_\alpha$ be calibrated on the base scores $s_{\text{glob}}$; the tempting deployment is to assert iff $s_{\text{acq}}(x,c)\ge\tau_\alpha$ (acquire evidence, keep the old threshold), and it is invalid: the threshold was calibrated on $s_{\text{glob}}$ but the decision uses $s_{\text{acq}}\ge s_{\text{glob}}$ (Eq.~\ref{eq:sacq}), so the acceptance region is enlarged for \emph{all} claims---hallucinated ones included---and the realized risk exceeds $\alpha$.

\begin{proposition}[Naive acquisition is anti-conservative]
\label{prop:naive}
Since $s_{\text{acq}}\ge s_{\text{glob}}$ pointwise, asserting on $\{s_{\text{acq}}\ge\tau_\alpha\}$ with $\tau_\alpha$ calibrated on $s_{\text{glob}}$ accepts a superset of $\{s_{\text{glob}}\ge\tau_\alpha\}$. The level-$\alpha$ guarantee certified for the latter does \emph{not} transfer to this larger set, whose selective risk need not be bounded by $\alpha$; whenever the extra accepted claims contain label-$0$ claims at a rate above $\alpha$, the realized risk exceeds $\alpha$.
\end{proposition}

Because $s_{\mathcal A}^{(B)}$ is monotone in $B$ (Prop.~\ref{prop:mono}, App.~A), more budget clears more true \emph{and} more false claims at a fixed bar, so the risk creeps up unless $\tau$ is re-raised. The fix is to make calibration and test traverse an identical pipeline: treat the acquisition policy $\mathcal{A}$ as part of the score (Eq.~\ref{eq:sacq}) and calibrate the threshold on the post-acquisition scores.

\begin{theorem}[Acquisition-adaptive validity]
\label{thm:valid}
Fix an acquisition policy $\mathcal{A}$ \emph{before observing the calibration sample}, mapping each $(x,c)$ to a post-acquisition score $s_{\mathcal{A}}(x,c)$ that depends only on $(x,c)$ and the model---independent of the entire calibration set, not merely its labels. If the calibration claims are i.i.d., then certifying a deterministic grid $\Lambda$ of thresholds (specified independently of the calibration sample) with a Bonferroni-corrected Clopper--Pearson bound (each at level $\delta/|\Lambda|$) and asserting on $\{s_{\mathcal{A}}\ge\hat\tau\}$ for the smallest certified $\hat\tau$ controls the selective risk at level $\alpha$ with probability $1-\delta$ on the test claims (App.~A). This is the \emph{exact} variant used in Table~\ref{tab:acq}; a held-out select-then-certify split is an equivalent alternative.
\end{theorem}

\emph{Proof sketch.} Applying the fixed map $\mathcal{A}$ to every claim preserves exchangeability of the pairs $(s_{\mathcal{A}}(x_i,c_i),y_i)$. The result is then standard split-conformal selective risk control with score $s_{\mathcal{A}}$ in place of $s$. $\square$

Theorem~\ref{thm:valid} instantiates a known principle---folding a fixed, label-independent transformation into the conformal pipeline keeps coverage \citep{shanmugam2025test, wang2026conformal}. We state it because the \emph{failure mode it rules out} (Proposition~\ref{prop:naive}) is, for visual acquisition, both easy to commit and empirically severe, as we show next.

Because it asks only that $\mathcal{A}$ be fixed independently of the calibration sample, Theorem~\ref{thm:valid} also permits acquisition that is \emph{sequential and adaptive within a claim}---``look, judge, look again, stop when confident'' is admissible with no correction (Prop.~\ref{prop:seq}), the stopping rule being absorbed into one deployed score. Tuning the policy class on calibration data (e.g.\ selecting $B$ or the band) is handled by a held-out split or Learn-then-Test over the finite policy grid \citep{angelopoulos2025learn}. The practical prescription is simply: \emph{calibrate on the scores you will actually deploy.}

\begin{algorithm}[t]
\caption{BCEA: calibration and deployment}
\label{alg:bcea}
\textbf{Input:} calibration claims $\{(x_i,c_i,y_i)\}_{i=1}^n$, level $\alpha$, confidence $\delta$, pre-fixed borderline band $[\tau^-,\tau^+]$, and $B$ crop views $x^{(1)},\dots,x^{(B)}$ of $x$
\begin{algorithmic}[1]
\STATE \textbf{function} \textsc{Score}$(x,c)$: \hfill // realizes the three-way answer/abstain/\emph{acquire}
\STATE \quad $\ell_0\leftarrow\ell(\varnothing,c)$;\ \ $s\leftarrow \ell(x,c)-\ell_0$ \hfill // base grounding
\STATE \quad \textbf{if} $s\in[\tau^-,\tau^+]$: \ \textbf{for} $b=1\dots B$: \ $s\leftarrow\max\big(s,\ \ell(x^{(b)},c)-\ell_0\big)$ \hfill // borderline $\Rightarrow$ acquire crop views; else answer/abstain directly
\STATE \quad \textbf{return} $s$
\STATE // \emph{Calibration} ($s_i\leftarrow\textsc{Score}(x_i,c_i)$ for all $i$)
\STATE \textbf{Exact:} $\tau\leftarrow\min\{t\in\Lambda:\ \overline{R}_{\delta/|\Lambda|}(\#\{y_i{=}0,s_i{\ge}t\},\,\#\{s_i{\ge}t\})\le\alpha\}$ over a fixed grid $\Lambda$
\STATE \textbf{Practical:} $\tau\leftarrow$ smallest $t$ with $\overline{R}_{\delta}(\#\{y_i{=}0,s_i{\ge}t\},\#\{s_i{\ge}t\})\le\alpha$ (scan calibration scores)
\STATE // \emph{Deployment} on a new claim $(x,c)$
\STATE \textbf{return} \textsc{assert} if $\textsc{Score}(x,c)\ge\tau$ \textbf{else} \textsc{abstain}
\end{algorithmic}
\end{algorithm}

Algorithm~\ref{alg:bcea} states BCEA in full. The \textsc{Score} function realizes the three-way decision by \emph{gating}: it runs the $B$ acquisition views only when the base grounding $s$ falls in the pre-fixed borderline band $[\tau^-,\tau^+]$ (a label-free, calibration-independent test), and answers/abstains on $s$ directly otherwise---so only borderline claims spend budget. The \textbf{Exact} branch then certifies a \emph{fixed}, label-free grid $\Lambda$ with a Bonferroni-corrected Clopper--Pearson bound, giving $\Pr[R(\hat\tau)>\alpha]\le\delta$ exactly (Thm~\ref{thm:valid}); the \textbf{Practical} branch scans the calibration scores for higher coverage at a mild validity cost. Either way $\tau$ is calibrated on \textsc{Score}, which already includes acquisition, so deployment and calibration see the same map---the acquisition-adaptive condition. Algorithm~\ref{alg:bcea} shows the \emph{crop} instantiation, whose acquisition score is $s_{\text{acq}}$ (Eq.~\ref{eq:sacq}); a claim-type \emph{intervention} uses the \emph{same} calibration but replaces the crop score with $s_T(x,c)=\ell(x,c)-\ell(T(x),c)$ (e.g.\ the horizontal flip), which is likewise a fixed label-free map and certifies identically---the flip signal is a likelihood \emph{difference}, not an extra view folded into the max.

\subsection{When Does Acquisition Help?}

Theorem~\ref{thm:valid} guarantees \emph{safety} but not \emph{improvement}. At the \emph{population} level, the optimal coverage under a selective-risk constraint is a point on the score's ROC curve (Prop.~\ref{prop:rocpoint}); ROC dominance is \emph{sufficient} for coverage dominance in general and, for concave or convexified ROCs, also \emph{necessary} (Thm.~\ref{thm:roc}), with a Blackwell-informativeness sufficient condition. Finite-sample \emph{certified} coverage can show small inversions from calibration noise.

This makes ``does acquisition help'' measurable and predicts the experiments: acquisition lifts AUROC $0.82\!\to\!0.86$ (grid) $\to\!0.88$ (guided, Table~\ref{tab:guided}), giving the coverage ordering no-acq.\ $<$ grid $<$ guided at most risk levels (finite-sample inversions possible at the smallest $\alpha$). The same lens explains our one deliberate \emph{negative}: per-claim budget \emph{allocation} is coverage-neutral here, because under a single global threshold an extra look raises true- and false-claim scores alike while labels are unknown at allocation time, so any label-agnostic allocation is to first order coverage-neutral (App.~A). We therefore do not claim allocation as a contribution.

\section{Experiments}

\paragraph{Setup.} We use LLaVA-1.5-7B \citep{liu2024improved} (main) and Qwen3-VL-7B \citep{bai2025qwen3} (generality), and construct claims from COCO val2017 ground-truth instance annotations \citep{lin2014coco} (no extra labeling). \emph{Existence} claims pair each image with present objects (true) and sampled absent objects (false); \emph{spatial} claims pair the two largest, horizontally separated objects as ``\emph{A} is to the left/right of \emph{B}'', with the swapped direction as the false claim. Each claim score is one forward pass; acquisition adds $B{=}5$ grid-crop passes. For the headline guarantee we use \emph{one claim per image} ($3{,}000$ images), so calibration claims are i.i.d.\ and the Clopper--Pearson binomial argument is exact; a $50/50$ calibration/test split is repeated $300$ times, $\delta{=}0.1$. We report two variants: \textbf{exact} (a Bonferroni-corrected certificate over a fixed threshold grid, which is provably valid) and \textbf{practical} (a single-set scan, which reaches higher coverage but is mildly anti-conservative). We report coverage and, as an empirical diagnostic, the \emph{violation frequency}: the fraction of splits whose held-out \emph{test} risk $\hat R_{\text{test}}(\hat\tau)$ exceeds $\alpha$. This estimates the realized risk and is \emph{distinct} from the theoretical false-certification probability $\Pr_{\mathcal D_{\text{cal}}}[R(\hat\tau)>\alpha]\le\delta$ that Theorem~\ref{thm:valid} bounds (a finite test set can exceed $\alpha$ by chance even when $R(\hat\tau)\le\alpha$). We report it as a complementary empirical diagnostic: a zero value indicates the certified thresholds are conservative on these splits, but it should not be read as an estimate of the theoretical false-certification probability. Multi-claim-per-image and per-model/POPE tables below use the practical variant with \emph{claim-level} splits, so multiple claims from one image can share calibration and test; those numbers are therefore empirical rather than exactly certified. The image-level, one-claim-per-image \emph{certified} counterpart---for all four backbones, with zero violations throughout---is Table~\ref{tab:backbone} in App.~C.

\subsection{Guaranteed Abstention Is Conservative}

Figure~\ref{fig:motivation} shows the \emph{empirical} (achievable, in-sample) risk--coverage frontier on $1{,}440$ balanced existence claims. On this frontier, to hold a hallucination rate $\le5\%$ the global-score filter asserts only $18\%$ of claims; even at a lenient $\le20\%$ it asserts $40\%$. This is the cost of abstention without acquisition, and the motivation for BCEA; the corresponding finite-sample \emph{certified} numbers (uniformly lower) are in Table~\ref{tab:acq}. The evidence-sufficiency score $s_{\text{glob}}$ already separates grounded from hallucinated claims (AUROC $0.838$) far better than raw claim likelihood ($0.808$), confirming that image-dependence is the right signal; acquisition raises it to $0.879$.

\subsection{Evidence Is Localized}

Before using acquisition we verify the grounding signal is spatially localized. On $800$ true existence claims, masking the object's ground-truth region lowers the claim log-likelihood by $0.109$ on average while masking an equal-area irrelevant region does not (paired difference $+0.159$, $67\%$ win rate, $p\ll10^{-10}$); an occlusion-sensitivity heatmap concentrates on the object (App.~C). Figure~\ref{fig:qual}(a) shows the effect on a real claim---the model's confidence depends specifically on the supporting pixels, the premise behind region- and intervention-based scoring.

\begin{figure}[t]
\centering
\vspace{-2mm}
\includegraphics[width=\columnwidth]{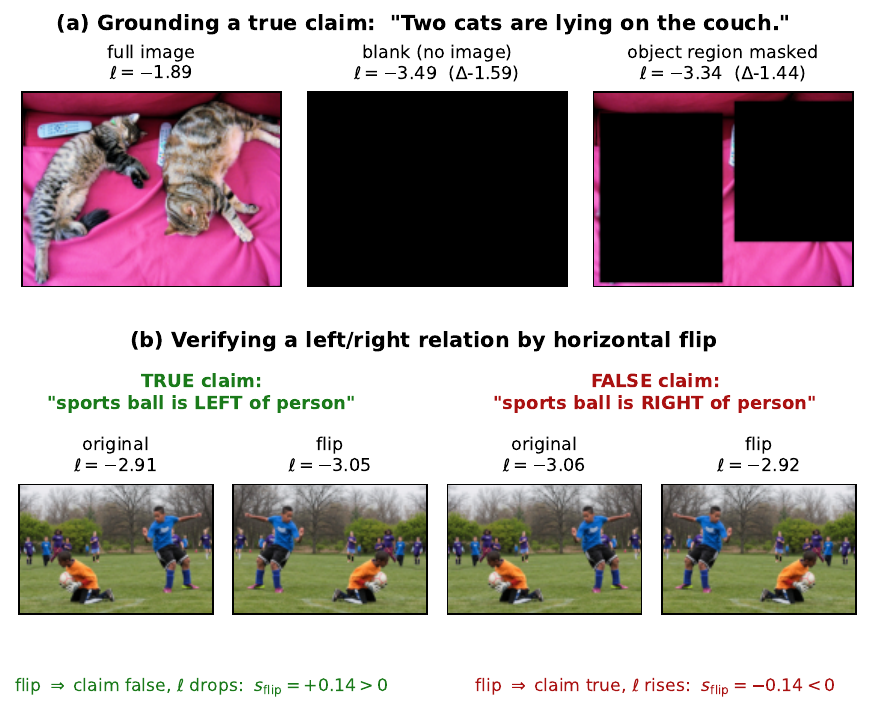}
\caption{Evidence sufficiency on real images (LLaVA-1.5, mean claim log-likelihood $\ell$). \textbf{(a)}~A true existence claim loses likelihood when the image is removed ($\Delta{-}1.6$) or the object region is masked ($\Delta{-}1.4$)---its confidence is grounded in those pixels. \textbf{(b)}~A left/right relation is verified by a horizontal flip: the \emph{correct} ``left of'' claim drops under flipping ($s_{\text{flip}}{=}{+}0.14$) while the \emph{swapped} ``right of'' claim rises ($s_{\text{flip}}{=}{-}0.14$). The global blank-image score, which leaves both relations equally likely, cannot provide this signal.}
\label{fig:qual}
\end{figure}

\subsection{Acquisition: Naive Breaks, BCEA Restores}

Table~\ref{tab:acq} is the central result ($3{,}000$ existence claims, one per image so the Clopper--Pearson argument is exact; $300$ random $50/50$ splits). \textbf{Exact} BCEA is proven valid by Theorem~\ref{thm:valid}, whose guarantee is the false-certification probability $\le\delta$. As a diagnostic we report the empirical violation frequency (test-set exceedance, defined above): it is $0.00$ at every $\alpha$. A finite test set could in principle exceed $\alpha$ by chance even at $R(\hat\tau){=}\alpha$, so this is not the theoretical quantity itself; but $0.00$ shows the certified threshold is \emph{conservative}---its realized risk stays below $\alpha$ on every split---consistent with the bound. The \emph{gated-selected} form (band and threshold chosen under one Bonferroni correction, itself valid) gives the strongest guaranteed coverage and clearly beats abstention at $\alpha{\ge}0.10$ ($0.32$ vs.\ $0.23$; $0.42$ vs.\ $0.28$), at $\approx2$ extra passes/claim. At $\alpha{=}0.05$ its $0.06$ sits below standalone No-Acq's $0.14$, but only because it searches a larger policy family and pays the Bonferroni price: under a \emph{matched} correction No-Acq also drops to $0.04$, and since a never-acquire band is a candidate, gating weakly dominates No-Acq at every $\alpha$ ($0.06/0.32/0.42$ vs.\ $0.04/0.23/0.24$); the sub-$5\%$ regime is simply evidence-limited. \textbf{Practical} BCEA reaches higher coverage ($0.32/0.43/0.52$) but is anti-conservative (violation $0.21/0.11/0.16$, up to $2\times$ target), so we do not call it guaranteed. The unguaranteed acquirers fail outright: \textsc{Naive} violates on \emph{every} split and the ReCoVERR-style \textsc{Emp.-thr.}\ about half. The precise claim is thus narrow: \emph{with a correct gating policy, certified acquisition improves guaranteed coverage at $\alpha{\ge}0.10$}, while the practical variant buys more coverage at a quantified validity cost.

\begin{table}[t]
\centering
\setlength{\tabcolsep}{5pt}
{\small\begin{tabular}{lccc}
\toprule
& $\alpha{=}0.05$ & $\alpha{=}0.10$ & $\alpha{=}0.20$ \\
method & cov/viol & cov/viol & cov/viol \\
\midrule
No-Acq \emph{(exact, $\delta/|\Lambda|$)} & $.14$/$.00$ & $.23$/$.00$ & $.28$/$.00$ \\
No-Acq \emph{(exact, matched corr.)} & $.04$/$.00$ & $.23$/$.00$ & $.24$/$.00$ \\
BCEA full-budget \emph{(exact)} & $.07$/$.00$ & $.22$/$.00$ & $.41$/$.00$ \\
\textbf{BCEA gated \emph{(exact)}} & $.06$/$.00$ & $\mathbf{.32}$/$.00$ & $\mathbf{.42}$/$.00$ \\
\midrule
BCEA \emph{(practical)} & $.32$/$\mathbf{.21}$ & $.43$/$.11$ & $.52$/$\mathbf{.16}$ \\
Naive & $.47$/$\mathbf{1.0}$ & $.60$/$\mathbf{1.0}$ & $.80$/$\mathbf{1.0}$ \\
Emp.-thr. & $.37$/$\mathbf{.55}$ & $.45$/$\mathbf{.47}$ & $.55$/$\mathbf{.48}$ \\
\bottomrule
\end{tabular}}
\caption{Acquisition on existence claims (LLaVA-1.5; $3{,}000$ one-per-image claims; coverage\,/\,violation frequency; $300$ splits; $\delta{=}0.1$). \emph{Exact} variants are certified by Theorem~\ref{thm:valid}; the $0$ empirical exceedance is a diagnostic, not the certified quantity. \emph{Gated} (band + threshold under one Bonferroni correction, $7$ bands $\times$ $17$ thresholds) gives the strongest guaranteed coverage and, containing a \emph{never-acquire} band, weakly dominates No-Acq \emph{at a matched correction}---so its lower $\alpha{=}0.05$ value is the price of a larger search, not harm. The \emph{practical} variant trades validity (anti-conservative, up to $2\times$ target) for coverage. \textbf{Bold} violations exceed $\delta$; \textsc{Naive}/\textsc{Emp.-thr.} fail.}
\label{tab:acq}
\end{table}

\begin{figure}[t]
\centering
\vspace{-4mm}
\includegraphics[width=0.5\textwidth]{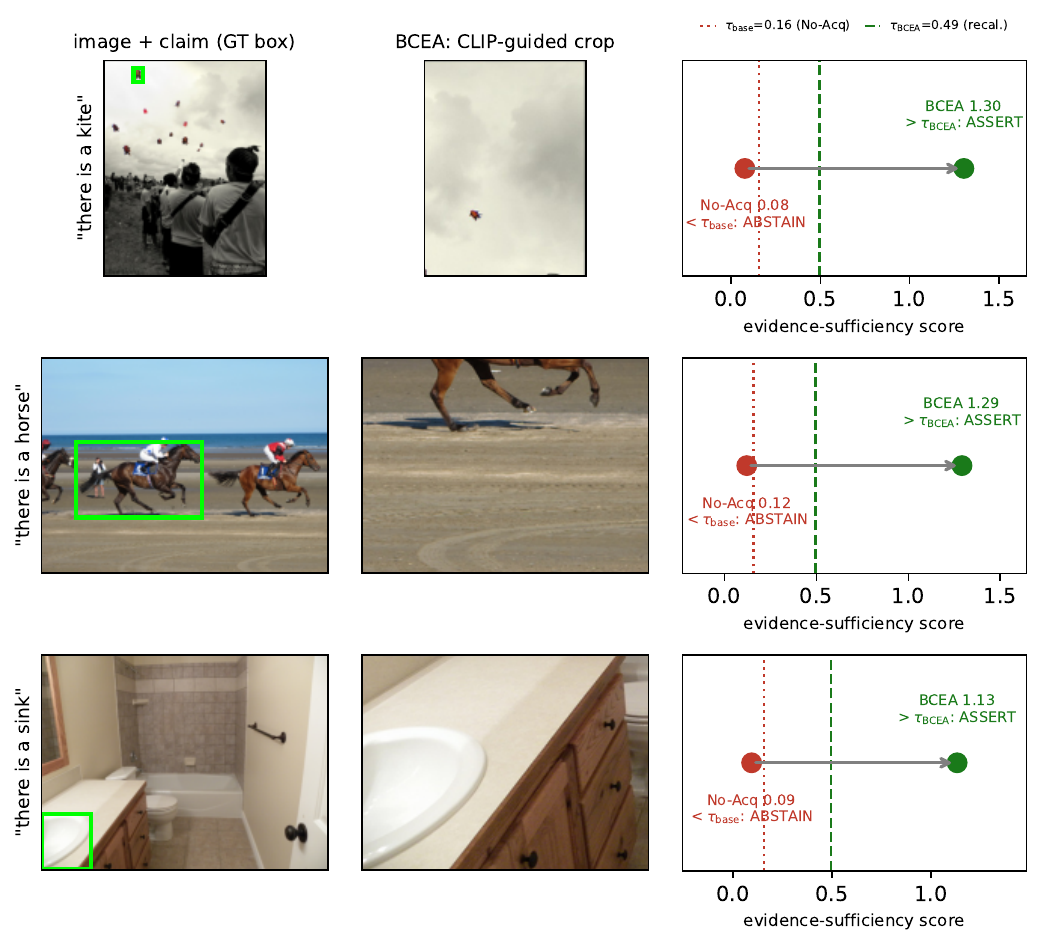}
\caption{\textbf{Rescued claims} (LLaVA-1.5, $\alpha{=}0.10$; GT box green). The two decisions use \emph{different} thresholds: No-Acq compares $s_{\text{glob}}$ to $\tau_{\text{base}}{=}0.16$ (red dotted), while BCEA re-calibrates and compares $s_{\text{acq}}$ to the \emph{higher} $\tau_{\text{BCEA}}{=}0.49$ (green dashed)---the recalibration that keeps the guarantee valid (Prop.~\ref{prop:naive}). Each base score abstains, yet a single CLIP-guided crop lifts it above $\tau_{\text{BCEA}}$ (assert)---claims a guaranteed filter would otherwise leave unanswered.}
\label{fig:rescue}
\vspace{-2mm}
\end{figure}

Figure~\ref{fig:rescue} makes the gain concrete: claims a guaranteed filter must abstain on are recovered, with evidence, after one targeted look (Table~\ref{tab:rescue} traces the decision flow, including a spatial case the global score cannot see but a flip intervention can).

\subsection{Structured Interventions for Relational Claims}

On $996$ balanced spatial-relation claims the global ungrounding score is near chance ($0.574$ AUROC): both objects are present, so blanking the image penalizes correct and swapped relations equally. A single horizontal-flip intervention---under which a correct ``left of'' claim becomes false---lifts verification to $0.765$ (Fig.~\ref{fig:qual}(b); Table~\ref{tab:spatial}, App.~C), with the predicted sign pattern (mean $s_{\text{flip}}$: $+0.034$ true vs.\ $-0.066$ false). Claim-type-specific interventions recover a signal that global ungrounding cannot see, and feed the same acquisition-and-recalibration machinery.

A valid intervention must satisfy two conditions: it must \emph{flip the claim's truth value} and \emph{keep the image in-distribution}. This is not automatic. The analogous vertical flip for above/below relations \emph{fails} (AUROC $0.50$): turning an image upside-down is out-of-distribution, so the model's likelihoods become uninformative rather than truth-reversing. Designing interventions is thus a modeling step---horizontal flip works for left/right precisely because a mirrored scene remains plausible---and a principled intervention library per claim type is an important direction the framework opens.

\subsection{Model-Guided Acquisition}

A uniform crop grid is blind to where the object is. We instead let CLIP \citep{radford2021learning} propose crops: among a multi-scale grid of candidate windows we keep the five whose CLIP embedding is most similar to ``a photo of a $\langle$object$\rangle$'' (Figure~\ref{fig:crop}), using no ground-truth location. Table~\ref{tab:guided} compares no acquisition, the uniform grid, and CLIP-guided crops on the same claims. Guided acquisition raises detection AUROC ($0.824\!\to\!0.862\!\to\!0.882$) and coverage at every level (e.g.\ $0.28\!\to\!0.33\!\to\!0.37$ at $\alpha{=}0.10$). The mechanism is exactly the one the grid missed: guided crops nearly \emph{double} the acquisition gain on \emph{small} objects ($0.25\!\to\!0.43$, Figure~\ref{fig:sizegain}), which a fixed grid tends to crop away, while leaving large-object gains unchanged. Allocating a fixed total budget greedily by marginal gain rather than uniformly makes no measurable difference here ($\le0.2$ coverage points), consistent with the roughly homogeneous per-claim gains on existence claims; we expect allocation to matter when claim difficulty is heterogeneous.

\begin{table}[t]
\centering
\setlength{\tabcolsep}{4pt}
{\small\begin{tabular}{lccc}
\toprule
 & no acq.\ & uniform grid & CLIP-guided \\
\midrule
AUROC & $0.824$ & $0.862$ & $\mathbf{0.882}$ \\
coverage $@\,\alpha{=}0.05$ & $0.19$ & $0.18$ & $\mathbf{0.22}$ \\
coverage $@\,\alpha{=}0.10$ & $0.28$ & $0.33$ & $\mathbf{0.37}$ \\
coverage $@\,\alpha{=}0.20$ & $0.39$ & $0.43$ & $\mathbf{0.47}$ \\
\bottomrule
\end{tabular}}
\caption{Model-guided vs.\ uniform acquisition (LLaVA, existence claims, practical variant, risk near $\alpha$). Targeting crops with CLIP improves both detection and coverage (practical variant).}
\label{tab:guided}
\end{table}

\subsection{Scaling to the POPE Benchmark}

To test BCEA beyond our COCO-constructed claims, we run it on the standard POPE hallucination benchmark \citep{li2023evaluating}---its \textsc{random}, \textsc{popular}, and \textsc{adversarial} splits, which sample absent objects with increasing difficulty---across \emph{four} open VLMs: LLaVA-1.5-7B \citep{liu2024improved}, Qwen3-VL-7B \citep{bai2025qwen3}, LLaVA-NeXT-7B \citep{liu2024llavanext}, and InternVL2-8B \citep{chen2024internvl} ($\sim$$11.5$k claims total). Table~\ref{tab:pope} reports detection AUROC before/after acquisition and (practical) coverage at $\alpha{=}0.10$. The pattern is uniform across all twelve settings: acquisition improves AUROC for every model and split (most strikingly LLaVA-NeXT on \textsc{random}, $0.56\!\to\!0.72$), and BCEA's coverage exceeds abstention everywhere---often two-to-fourfold---while its $90$th-percentile risk tracks the target. Coverage falls from \textsc{random} to \textsc{adversarial} exactly as the benchmark's difficulty rises, and acquisition helps most where the base score is weakest. We stress that Table~\ref{tab:pope} is an \emph{empirical evaluation}, not a finite-sample certificate: its splits are claim-level, so multiple claims from one image can straddle calibration and test, and the practical scan is mildly anti-conservative. The certified, image-level, one-claim-per-image counterpart---which \emph{holds for all four backbones with zero violations}---is Table~\ref{tab:backbone} (App.~C).

\begin{table}[t]
\centering
\setlength{\tabcolsep}{4pt}
{\small\begin{tabular}{llcccc}
\toprule
& & \multicolumn{2}{c}{AUROC} & \multicolumn{2}{c}{coverage @ $\alpha{=}0.10$}\\
\cmidrule(lr){3-4}\cmidrule(lr){5-6}
model & split & glob & +acq & No-Acq & BCEA \\
\midrule
\multirow{3}{*}{LLaVA-1.5} & random      & $.82$ & $.87$ & $.28$ & $\mathbf{.36}$ \\
                           & popular     & $.77$ & $.81$ & $.06$ & $\mathbf{.14}$ \\
                           & adversarial & $.73$ & $.77$ & $.03$ & $\mathbf{.07}$ \\
\multirow{3}{*}{Qwen3-VL}& random      & $.66$ & $.75$ & $.05$ & $\mathbf{.13}$ \\
                           & popular     & $.70$ & $.78$ & $.09$ & $\mathbf{.18}$ \\
                           & adversarial & $.64$ & $.72$ & $.02$ & $\mathbf{.09}$ \\
\multirow{3}{*}{LLaVA-NeXT}& random      & $.56$ & $.72$ & $.09$ & $\mathbf{.18}$ \\
                           & popular     & $.59$ & $.72$ & $.03$ & $\mathbf{.14}$ \\
                           & adversarial & $.57$ & $.69$ & $.02$ & $\mathbf{.03}$ \\
\multirow{3}{*}{InternVL2} & random      & $.75$ & $.82$ & $.04$ & $\mathbf{.08}$ \\
                           & popular     & $.70$ & $.82$ & $.03$ & $\mathbf{.16}$ \\
                           & adversarial & $.64$ & $.71$ & $.00$ & $\mathbf{.02}$ \\
\bottomrule
\end{tabular}}
\caption{BCEA on POPE, four VLMs $\times$ three splits ($\sim$$11.5$k claims). Acquisition improves AUROC (glob\,$\to$\,+acq) and BCEA beats abstention everywhere; its $90$th-percentile risk tracks $\alpha{=}0.10$ to within $\sim0.03$. Difficulty: random\,$<$\,popular\,$<$\,adversarial.}
\label{tab:pope}
\end{table}

\subsection{Free-Form Generation}

The experiments so far verify \emph{probed} claims. To test the motivating setting directly, we run the full pipeline end-to-end---free-form generation $\to$ claim decomposition $\to$ acquisition $\to$ calibrated assertion: LLaVA-1.5 writes an open-ended description of an image (``Describe this image in detail.''), a parser extracts the COCO categories it mentions as atomic existence claims, COCO ground truth judges each claim, and BCEA acquires evidence and certifies it. Claims are thus \emph{self-generated}, not supplied by us. To keep calibration i.i.d.\ we retain \emph{one claim per image, chosen uniformly at random} among the mentioned categories (not the easiest or hardest), so the selection is label-independent; across $1{,}563$ such images $20.6\%$ of the claims LLaVA volunteers are hallucinated---the price of believing the caption outright.

BCEA transfers (Table~\ref{tab:freeform}). Acquisition lifts detection AUROC from $0.769$ to $0.820$, and under the \emph{exact} certificate at $\alpha{=}0.10$ it asserts $34\%$ of self-generated claims with a certified $\le10\%$ hallucination rate, against $24\%$ for guaranteed abstention---a $10$-point gain with zero empirical violations. At $\alpha{=}0.20$, close to the base rate, most claims can already be certified. This does not close the gap entirely---the parser handles object mentions only, and attributes, relations and counts inside a caption remain out of scope---but it shows the mechanism is not an artifact of hand-constructed probes.

\begin{table}[t]
\centering
\setlength{\tabcolsep}{5pt}
{\small\begin{tabular}{lccc}
\toprule
$\alpha$ & No-Acq \emph{(exact)} & BCEA \emph{(exact)} & BCEA \emph{(practical)} \\
\midrule
$0.05$ & $.00$/$.00$ & $.00$/$.00$ & $.30$/$\mathbf{.16}$ \\
$0.10$ & $.24$/$.00$ & $\mathbf{.34}$/$.00$ & $.67$/$\mathbf{.17}$ \\
$0.20$ & $.75$/$.00$ & $.74$/$.01$ & $.94$/$\mathbf{.17}$ \\
\bottomrule
\end{tabular}}
\caption{\textbf{Free-form claims} that LLaVA-1.5 generates itself ($1{,}563$ captions, one parsed claim each; $20.6\%$ hallucinated; coverage / violation). Under the exact certificate BCEA asserts $34\%$ of self-generated claims at a certified $\le10\%$ hallucination rate vs.\ $24\%$ for abstention. AUROC $0.769\!\to\!0.820$ with acquisition.}
\label{tab:freeform}
\end{table}

\emph{Further analyses appear in App.~C:} adaptive within-claim stopping, exact-certified guided acquisition with a compute-matched active-perception comparison, per-backbone image-level certification (Table~\ref{tab:backbone}), a stricter-$\delta$ sweep, a \textsc{ConfLVLM} scoring head-to-head, and additional visualizations.

\section{Conclusion}

Guaranteed selective prediction makes LVLMs reliable but over-cautious. BCEA shows the model can \emph{look again} rather than give up, acquiring cheap visual evidence to rescue claims, provided the acquisition policy is folded into the score and the threshold re-calibrated on what is deployed. Naive acquisition silently breaks the guarantee; re-calibration restores it and still improves coverage. With a validated claim-specific intervention that sees what global ungrounding cannot, this turns a binary answer/abstain filter into a budgeted, guaranteed, evidence-seeking one.

\appendix
\numberwithin{table}{section}
\numberwithin{figure}{section}
\numberwithin{equation}{section}
\numberwithin{theorem}{section}
\numberwithin{proposition}{section}
\numberwithin{lemma}{section}
\numberwithin{corollary}{section}
\numberwithin{algorithm}{section}
\section{Proofs}

\paragraph{Setup and notation.}
Let $Z_i=(X_i,C_i,Y_i)$, $i=1,\dots,n$, be calibration claims and $Z=(X,C,Y)$ a test claim, all i.i.d.\ from a distribution $P$ on $\mathcal{X}\times\mathcal{C}\times\{0,1\}$, with $Y{=}1$ marking a supported claim. Model parameters $\theta$ are fixed and estimated on data independent of $\{Z_i\}\cup\{Z\}$. A \emph{score} is any measurable $s:\mathcal{X}\times\mathcal{C}\to\mathbb{R}$ (it may use $\theta$ but not the labels $\{Y_i\}$); write $V_i=s(X_i,C_i)$ and $V=s(X,C)$. A selective rule asserts iff $V\ge\tau$; its coverage is $\mathrm{cov}(\tau)=\Pr[V\ge\tau]$ and its \emph{selective risk} is $R_s(\tau)=\Pr[Y{=}0\mid V\ge\tau]$.

\begin{lemma}[Label-independent maps transport exchangeability]
\label{lem:transport}
If $s$ is fixed independently of the calibration and test samples and applied pointwise to each $(X_i,C_i)$ (in particular it does not read $\{Y_i\}$, nor any other $X_j,C_j$---e.g.\ a data-derived quantile), then $\{(V_i,Y_i)\}_{i=1}^n\cup\{(V,Y)\}$ are i.i.d.\ (hence exchangeable), and for any fixed $\tau$, conditional on $\{V_j\ge\tau\}$ the accepted labels are i.i.d.\ $\mathrm{Bernoulli}(R_s(\tau))$. Any internal randomness in $s$ (e.g.\ random crops) must be drawn from a fixed distribution independently of the data and used identically in calibration and deployment.
\end{lemma}
\emph{Proof.} Each pair is the image of $Z_i$ (resp.\ $Z$) under the fixed measurable map $(x,c,y)\mapsto(s(x,c),y)$; pushing an i.i.d.\ sample through one fixed map yields an i.i.d.\ sample. Independence across $i$ then makes the events $\{V_i\ge\tau\}$ independent, and on the sub-population $\{V_i\ge\tau\}$ each $Y_i$ is Bernoulli with parameter $\Pr[Y{=}0\mid V\ge\tau]=R_s(\tau)$. $\square$

\emph{Proof of Theorem~\ref{thm:valid}.} For a threshold $\tau$ let $n(\tau)=\#\{i:V_i\ge\tau\}$ and $k(\tau)=\#\{i:V_i\ge\tau,\,Y_i{=}0\}$, and let $\overline{R}_{\delta'}(\tau)=\mathrm{Beta}^{-1}_{1-\delta'}\!\big(k(\tau)+1,\,n(\tau)-k(\tau)\big)$ be the Clopper--Pearson upper risk bound at level $\delta'$, with the convention $\overline{R}_{\delta'}(\tau)=1$ when $n(\tau)=0$ or $k(\tau)=n(\tau)$ (so such $\tau$ are never certified). Fix a deterministic grid $\Lambda$ that does not depend on the calibration sample. By Lemma~\ref{lem:transport}, for each fixed $\tau\in\Lambda$, $k(\tau)\mid n(\tau)\sim\mathrm{Binomial}(n(\tau),R(\tau))$, so $\overline{R}_{\delta/|\Lambda|}(\tau)$ satisfies $\Pr[\overline{R}_{\delta/|\Lambda|}(\tau)<R(\tau)]\le\delta/|\Lambda|$; equivalently $\{\overline{R}_{\delta/|\Lambda|}(\tau)\le\alpha\}$ falsely certifies a genuinely unsafe $\tau$ ($R(\tau)>\alpha$) with probability $\le\delta/|\Lambda|$. A union bound over the $|\Lambda|$ thresholds gives $\Pr[\exists\tau\in\Lambda:\ \overline{R}_{\delta/|\Lambda|}(\tau)\le\alpha\ \text{but}\ R(\tau)>\alpha]\le\delta$. Hence $\hat\tau=\min\{\tau\in\Lambda:\overline{R}_{\delta/|\Lambda|}(\tau)\le\alpha\}$---the \emph{smallest} certified threshold, which maximizes coverage, or ``abstain'' if the set is empty---satisfies $R(\hat\tau)\le\alpha$ with probability $\ge1-\delta$; no ordering or monotonicity is invoked. The only model-specific ingredient is that $s_{\mathcal A}$ obeys Lemma~\ref{lem:transport} (acquisition folded into one label-free score), the principle of \citet{shanmugam2025test, wang2026conformal}. An equivalent construction selects $\tau$ on a held-out split $S_1$ and certifies it with a single Clopper--Pearson bound on independent $S_2$. The \emph{practical} variant instead scans the calibration scores at level $\delta$ (no Bonferroni), reusing one set to select and certify $\tau$; it is mildly anti-conservative (violation frequency around or above $\delta$, App.~B) and is what the practical rows of Table~\ref{tab:acq} and the per-model/POPE tables report. $\square$

\emph{Remark (selecting the policy from calibration).} Theorem~\ref{thm:valid} fixes a single policy $\mathcal A$ in advance, but the same argument certifies a \emph{deterministic finite policy class} $\{\mathcal A\}$ (e.g.\ the seven pre-fixed borderline bands): certify every pair $(\mathcal A,\tau)\in\{\mathcal A\}\times\Lambda$ with a Clopper--Pearson bound at level $\delta/(|\{\mathcal A\}|\,|\Lambda|)$ and select any certified pair. A union bound over the $|\{\mathcal A\}|\,|\Lambda|$ hypotheses preserves the $1-\delta$ guarantee even though the selected band is chosen \emph{on} the calibration sample---each candidate policy is itself fixed a priori. This is the \emph{gated-selected} variant of Table~\ref{tab:acq}.

\emph{Proof of Proposition~\ref{prop:naive}.}
By construction $s_{\text{acq}}(x,c)=\max\!\big(\ell(x),\max_b\ell(x^{(b)})\big)-\ell(\varnothing)\ge \ell(x)-\ell(\varnothing)=s_{\text{glob}}(x,c)$ pointwise, so $\{s_{\text{acq}}\ge\tau_\alpha\}\supseteq\{s_{\text{glob}}\ge\tau_\alpha\}$ for any $\tau_\alpha$. Suppose $\tau_\alpha$ was calibrated on $s_{\text{glob}}$, so (Theorem~\ref{thm:valid}) $R_{s_{\text{glob}}}(\tau_\alpha)\le\alpha$. Decompose the deployed acceptance event as the disjoint union $E\cup E'$ with $E=\{s_{\text{glob}}\ge\tau_\alpha\}$ and $E'=\{s_{\text{acq}}\ge\tau_\alpha,\,s_{\text{glob}}<\tau_\alpha\}$. Writing $w=\Pr[E']/\Pr[E\cup E']$, the deployed risk is the mixture
\[
R_{s_{\text{acq}}}(\tau_\alpha)=(1-w)\,R_{s_{\text{glob}}}(\tau_\alpha)+w\,\Pr[Y{=}0\mid E'].
\]
The calibration certifies only the first term. If the error rate on the \emph{newly admitted} claims satisfies $\Pr[Y{=}0\mid E']>\alpha$, then $R_{s_{\text{acq}}}(\tau_\alpha)>\alpha$ whenever $w>0$. The set $E'$ consists exactly of claims an ungrounded full-image score rejected but some crop pushed over $\tau_\alpha$; for hallucinated claims a crop can raise $\ell$ spuriously, so $\Pr[Y{=}0\mid E']$ is typically large, and the guarantee is violated---as observed (Table~\ref{tab:acq}, ``Naive''). $\square$

\begin{proposition}[Sequential within-claim acquisition]
\label{prop:seq}
Let $\mathcal{A}$ form views one at a time and decide whether to acquire another or stop, where each decision (the next view and the stopping time $\kappa\le B$) is a function of the current claim's observations alone. Then $s_{\mathcal{A}}$ is a fixed measurable function of $(x,c)$, and the conclusion of Theorem~\ref{thm:valid} holds verbatim.
\end{proposition}

\emph{Proof.}
Let the policy reveal views $x^{(1)},x^{(2)},\dots$ and, after each, decide whether to stop using only $\{(x,c),x^{(1)},\dots,x^{(b)},\theta\}$; let $\kappa\le B$ be the (data-determined) stopping index and $s_{\mathcal A}=\max(\ell(x),\max_{b\le\kappa}\ell(x^{(b)}))-\ell(\varnothing)$. Because every decision and the final aggregation are deterministic functions of the current claim and $\theta$ alone, $s_{\mathcal A}=g(x,c;\theta)$ for a single fixed measurable $g$ that does not read $\{Y_i\}$. Lemma~\ref{lem:transport} and Theorem~\ref{thm:valid} therefore apply verbatim. (If $g$ uses internal randomness, e.g.\ random crops, draw it from a fixed distribution independent of the data, or fix a common seed, so that the same $g$ acts on calibration and test claims.) $\square$

\begin{proposition}[Monotone acquisition]
\label{prop:mono}
Let $s_{\mathcal A}^{(B)}(x,c)=\max(\ell(x),\max_{b\le B}\ell(x^{(b)}))-\ell(\varnothing)$. Then $s_{\mathcal A}^{(B)}$ is non-decreasing in $B$ pointwise, so at any \emph{fixed} threshold $\tau$ the accepted set is nested in $B$. Hence both the coverage $\Pr[s_{\mathcal A}^{(B)}\ge\tau]$ and the \emph{false-acceptance probability} $\Pr[y{=}0,\ s_{\mathcal A}^{(B)}\ge\tau]$ are non-decreasing in $B$.
\end{proposition}

\emph{Proof.}
For each $(x,c)$, $s_{\mathcal A}^{(B+1)}=\max\!\big(s_{\mathcal A}^{(B)},\,\ell(x^{(B+1)})-\ell(\varnothing)\big)\ge s_{\mathcal A}^{(B)}$, so the map $B\mapsto s_{\mathcal A}^{(B)}$ is pointwise non-decreasing and the acceptance events are nested: $\{s_{\mathcal A}^{(B)}\ge\tau\}\subseteq\{s_{\mathcal A}^{(B+1)}\ge\tau\}$. Monotonicity of measure gives $\mathrm{cov}_B(\tau)=\Pr[s_{\mathcal A}^{(B)}\ge\tau]\le\mathrm{cov}_{B+1}(\tau)$ and, intersecting with the fixed event $\{Y{=}0\}$, $\Pr[Y{=}0,\,s_{\mathcal A}^{(B)}\ge\tau]\le\Pr[Y{=}0,\,s_{\mathcal A}^{(B+1)}\ge\tau]$. The \emph{conditional} risk $R(\tau)=\Pr[Y{=}0\mid s_{\mathcal A}^{(B)}\ge\tau]$, a ratio of two non-decreasing quantities, is not monotone in general; the newly admitted claims may be cleaner or dirtier than average. This is exactly why a budget increase requires re-solving for $\tau$ rather than reusing it: at fixed $\tau$ both true and false acceptances can only grow. $\square$

\begin{proposition}[Coverage at fixed risk is an ROC point]
\label{prop:rocpoint}
Assume $\alpha<1-\pi$ (below the negative base rate; otherwise asserting everything already meets the risk target, and the optimum is trivially full coverage) and that the ROC is concave (equivalently, thresholds may be randomized so the achievable region is the convex hull). Then the maximum coverage of a selective rule on $S$ subject to selective risk $\le\alpha$ is attained at the ROC point on the ray $\mathrm{FPR}/\mathrm{TPR}=\rho(\alpha):=\tfrac{\pi\alpha}{(1-\pi)(1-\alpha)}$, and equals $C_S(\alpha)=\big(\pi+(1-\pi)\rho(\alpha)\big)\,\mathrm{TPR}^\star_S(\alpha)$. Without concavity, the optimum is a vertex of the ROC's upper boundary and the ray characterization holds up to that convexification.
\end{proposition}

\emph{Proof.}
With $T=\mathrm{TPR}_S(\tau)$, $F=\mathrm{FPR}_S(\tau)$, coverage is $C=\pi T+(1-\pi)F$ and selective risk is $R=\tfrac{(1-\pi)F}{\pi T+(1-\pi)F}$. Solving $R\le\alpha$ gives $(1-\pi)F(1-\alpha)\le\alpha\pi T$, i.e.\ $F/T\le\rho(\alpha)=\tfrac{\pi\alpha}{(1-\pi)(1-\alpha)}$: the feasible operating points are those on or below the ray of slope $\rho(\alpha)$ in ROC space. On any ray $F=\rho T$, $C=T(\pi+(1-\pi)\rho)$ is strictly increasing in $T$, and $C$ increases up-right; for a concave ROC the maximizer over the feasible region $\{F/T\le\rho(\alpha)\}$ thus lies on the ray boundary $F/T=\rho(\alpha)$ at the largest attainable $T$, giving $C_S(\alpha)=(\pi+(1-\pi)\rho(\alpha))\,\mathrm{TPR}^\star_S(\alpha)$. (For a non-concave ROC the optimum is a vertex of its upper concave envelope; randomizing thresholds attains the envelope.) $\square$

\begin{theorem}[ROC dominance and coverage]
\label{thm:roc}
For scores $S$ (pre-acquisition) and $S'=s_{\mathcal A}$ (post-acquisition), ROC dominance is \emph{sufficient} in general and, for concave (or convexified) ROCs, also \emph{necessary}: if $\mathrm{ROC}_{S'}$ dominates $\mathrm{ROC}_{S}$ pointwise then $C_{S'}(\alpha)\ge C_{S}(\alpha)$ for \emph{every} $\alpha\in(0,1)$, and under concave/convexified ROCs the converse holds too. A sufficient condition for ROC dominance is that $S'$ be more informative about $Y$ than $S$ in the Blackwell sense (its class-conditional likelihood ratio is a monotone refinement of that of $S$).
\end{theorem}

\emph{Proof.}
($\Leftarrow$, holds generally) Suppose $\mathrm{ROC}_{S'}\!\ge\!\mathrm{ROC}_{S}$ pointwise. Fix $\alpha$ and let $(F,T)$ be \emph{any} feasible optimal operating point for $S$. At the same $F$, $S'$ attains $T'=\mathrm{ROC}_{S'}(F)\ge T$; the point $(F,T')$ has risk $\tfrac{(1-\pi)F}{\pi T'+(1-\pi)F}\le\tfrac{(1-\pi)F}{\pi T+(1-\pi)F}\le\alpha$ (feasible) and coverage $\pi T'+(1-\pi)F\ge C_S(\alpha)$; hence $C_{S'}(\alpha)\ge C_S(\alpha)$. ($\Rightarrow$, under concave ROCs) If $\mathrm{ROC}_{S'}\!\not\ge\!\mathrm{ROC}_S$ there is $F_0$ with $\mathrm{ROC}_{S'}(F_0)<\mathrm{ROC}_S(F_0)$; for concave ROCs, choosing $\alpha$ with $\rho(\alpha)=F_0/\mathrm{ROC}_S(F_0)$ makes $(F_0,\mathrm{ROC}_S(F_0))$ the $S$-optimum (Prop.~\ref{prop:rocpoint}), which $S'$ cannot match, so $C_{S'}(\alpha)<C_S(\alpha)$; without concavity the converse can fail. Finally, if $S'$ is Blackwell-more-informative than $S$ for $Y$, its ROC dominates (Neyman--Pearson), giving the sufficient condition. $\square$

\paragraph{Allocation (not claimed; for completeness).}
Were per-claim gains $g_i$ non-decreasing and concave, $\max\sum_i g_i(b_i)$ s.t.\ $\sum_i b_i\le B_{\text{tot}}$ would be solved by greedily incrementing $\arg\max_i\Delta_i(b_i)$, $\Delta_i(b)=g_i(b{+}1)-g_i(b)$: for any optimal $b^{*}\ne b^{g}$ pick $p,q$ with $b^{g}_p>b^{*}_p$, $b^{g}_q<b^{*}_q$; since greedy preferred $p$, $\Delta_p(b^{*}_p)\ge\Delta_q(b^{*}_q)$, and moving a unit from $q$ to $p$ changes the objective by $\Delta_p(b^{*}_p)-\Delta_q(b^{*}_q-1)\ge0$ (concavity), so $b^{g}$ is optimal, with KKT level $\lambda^\star$ s.t.\ $\Delta_i(b^{g}_i{-}1)\ge\lambda^\star\ge\Delta_i(b^{g}_i)$. As the main text reports, the premise (predictable heterogeneous gains) does not hold for our existence claims, so this allocation does not beat uniform in practice; we include it only to delimit when it could. $\square$

\section{Implementation Details}

\paragraph{Models.} LLaVA-1.5-7B and Qwen3-VL-7B-Instruct (HuggingFace checkpoints), fp16, single A6000 each; inference only, no fine-tuning.

\paragraph{Scoring.} For claim $c$ and visual input $v$ we feed the chat prompt ``USER: \texttt{<image>} What is in this image? ASSISTANT: $c$'' and average the token log-probabilities of the claim span $c$, locating the span by tokenizing the prompt prefix alone. The blank reference $\varnothing$ is an all-black image of the same size; crops are five overlapping $62\%$ windows (four corners + center); region masks and one-instance masks black out ground-truth COCO boxes; the flip intervention is a horizontal mirror.

\paragraph{Claims.} Built from COCO val2017 instance annotations. \emph{Existence}: present categories (true) vs.\ sampled absent categories (false). \emph{Spatial}: the two largest distinct categories whose centroids differ by $>0.25$ of the image width, phrased ``the $A$ is to the left/right of the $B$''; the swapped direction is the false claim. \emph{Count}: ``there are exactly $k$ $A$s'' with $k$ the true count (true) or $k\pm1$ (false), restricted to $2$--$6$ instances.

\paragraph{Conformal procedure.} $50/50$ calibration/test split, $300$ repetitions, $\delta{=}0.1$. The \emph{exact} variant (Table~\ref{tab:acq} headline) certifies a grid $\Lambda$ of $17$ thresholds evenly spaced over the \emph{pre-specified} interval $[-4,4]$ (the score is a mean per-token log-likelihood difference, which lies in this range; the grid is fixed a priori, not read from the calibration scores) with a Bonferroni-corrected Clopper--Pearson bound at level $\delta/|\Lambda|$, taking the smallest certified threshold (App.~A); on one-claim-per-image data its violation frequency is $0.00$ at every $\alpha$. The \emph{gated-selected} rule additionally Bonferroni-corrects over a pre-fixed set of numeric bands, so band selection is valid. The \emph{practical} variant scans the (label-free) calibration scores at level $\delta$, reusing one set to select and certify; it is anti-conservative (violation $0.11$--$0.21$, up to $2\times$ target) and is what the practical row of Table~\ref{tab:acq} and the per-model/POPE tables report (claim-level splits). The $300$ splits reuse one dataset, so they estimate rather than independently replicate the guarantee.

\paragraph{Calibration-set size.} Coverage improves monotonically with calibration size (Table~\ref{tab:calsize}), the practical lever in the conservative small-$\alpha$ regime; the slight risk inflation at $80\%$ is small-test-set noise.

\begin{table}[h]
\centering
{\small\begin{tabular}{lccccc}
\toprule
calibration \% & $20$ & $35$ & $50$ & $65$ & $80$\\
\midrule
coverage ($\alpha{=}0.10$) & $0.30$ & $0.33$ & $0.34$ & $0.35$ & $0.36$\\
risk ($90$th pct) & $0.09$ & $0.10$ & $0.10$ & $0.11$ & $0.12$\\
\bottomrule
\end{tabular}}
\caption{Calibration-set-size effect (LLaVA, existence). More calibration data buys coverage at controlled risk.}
\label{tab:calsize}
\end{table}

\paragraph{Score ablation.} On existence claims, claim-detection AUROC rises from raw claim likelihood ($0.808$) to global ungrounding $s_{\text{glob}}$ ($0.838$) to the acquisition score $s_{\text{acq}}$ ($0.879$), confirming that image-dependence, and then acquisition, each add signal.

\paragraph{Compute.} All experiments are inference-only on RTX A6000 GPUs; the full study (four models, $\sim$$11.5$k POPE claims plus COCO-constructed existence/spatial/attribute claims, with five-crop acquisition and calibration sweeps) runs in tens of GPU-hours.

\section{Additional Experiments}

\subsection{Per-Backbone Image-Level Certification}
Table~\ref{tab:pope} is claim-level and empirical; here we certify all four backbones under the \emph{exact}, image-level, one-claim-per-image procedure of Table~\ref{tab:acq} (POPE-\textsc{random} existence claims re-scored with the image id stored, one claim kept per image so calibration is i.i.d.). Table~\ref{tab:backbone} shows the guarantee \emph{holds for every backbone}: the certificate (Theorem~\ref{thm:valid}) applies, and the empirical test-set exceedance is $0.00$ at every $\alpha$; acquisition lifts detection AUROC for all four ($0.57$--$0.85\to0.73$--$0.89$), and BCEA's guaranteed coverage meets or exceeds abstention at every setting (e.g.\ at $\alpha{=}0.20$, LLaVA-NeXT $0.01\to0.18$, InternVL2 $0.00\to0.17$). The per-image sample is small ($N\approx500$), so the conservative certificate leaves coverage near zero at the tightest $\alpha$; the certified \emph{improvement} from acquisition is clearest at $\alpha{=}0.20$. This is the image-level, certified counterpart of the empirical POPE results.

\begin{table}[htbp]
\centering
\setlength{\tabcolsep}{4pt}
{\small\begin{tabular}{lccccc}
\toprule
& & \multicolumn{2}{c}{AUROC} & \multicolumn{2}{c}{No-Acq / BCEA cov}\\
\cmidrule(lr){3-4}\cmidrule(lr){5-6}
backbone & $N$ & glob & +acq & $\alpha{=}.10$ & $\alpha{=}.20$ \\
\midrule
LLaVA-1.5   & $500$ & $.85$ & $.89$ & $.00/.04$ & $.15/\mathbf{.25}$ \\
Qwen3-VL    & $487$ & $.67$ & $.76$ & $.00/.00$ & $.01/\mathbf{.04}$ \\
LLaVA-NeXT  & $487$ & $.57$ & $.73$ & $.00/.01$ & $.01/\mathbf{.18}$ \\
InternVL2   & $500$ & $.73$ & $.84$ & $.00/.00$ & $.00/\mathbf{.17}$ \\
\bottomrule
\end{tabular}}
\caption{\textbf{Image-level, one-claim-per-image \emph{exact} certification, all four backbones} (POPE-\textsc{random} existence, $300$ splits, $\delta{=}0.1$). The certificate (Theorem~\ref{thm:valid}) applies to every backbone, and the empirical test-set exceedance is $\mathbf{0.00}$ at every $\alpha$. Acquisition lifts AUROC and guaranteed coverage throughout; small $N$ makes the tightest-$\alpha$ coverage conservative.}
\label{tab:backbone}
\end{table}

\subsection{Coverage Scales with the Budget}

Figure~\ref{fig:budget} sweeps the acquisition budget $B$ under the \emph{practical} variant, re-calibrating at each $B$. Coverage generally increases with budget---from $0.29$ at $B{=}0$ to $0.35$ at $B{=}5$ for $\alpha{=}0.10$---while the violation frequency stays near $\delta$. (This is the practical regime; under the exact certificate the effect reverses---targeted acquisition beats full budget, Figure~\ref{fig:gate}.) It need not be monotone: at $\alpha{=}0.05$ the recalibrated coverage dips slightly after $B{=}1$, because Proposition~\ref{prop:mono} bounds only \emph{fixed}-threshold quantities, not the recalibrated coverage, and finite-sample recalibration adds noise. \emph{Where} to spend a fixed total budget across claims, by contrast, does not matter here: greedy and uniform allocation are statistically indistinguishable (both $\approx0.34$, within CI), for the structural reason given above.

\paragraph{Adaptive stopping makes the budget real.} The fixed budget can be spent \emph{adaptively}: acquire guided crops one at a time and stop as soon as the running score clears the deployment threshold $\hat\tau$ (Proposition~\ref{prop:seq} makes this within-claim stopping valid at no correction). Because the running max is monotone, this makes \emph{identical} assert/abstain decisions to full-budget acquisition---same coverage, same zero violations under the exact certificate---while cutting the average number of acquisition passes from $5$ to $3.96$ at $\alpha{=}0.10$ and $3.22$ at $\alpha{=}0.20$ ($21$--$36\%$ fewer; at the near-abstain $\alpha{=}0.05$ almost every claim gets a full look, so the saving is small). Claims that clear early stop early and only hard claims consume the full allowance, so ``budgeted'' denotes an adaptively spent budget, not merely a fixed cap.

\begin{figure}[htbp]
\centering
\includegraphics[width=\columnwidth]{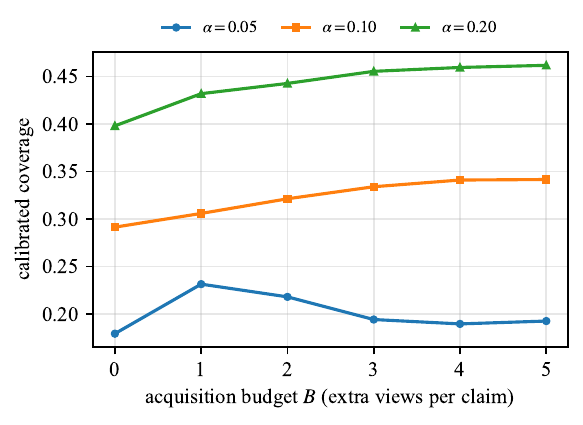}
\caption{\emph{Practical} variant: calibrated coverage vs.\ acquisition budget $B$, re-calibrating at each $B$. Coverage generally grows with budget (not monotonically---see $\alpha{=}0.05$), while the violation frequency stays near $\delta$. Under the \emph{exact} certificate the effect reverses (Figure~\ref{fig:gate}).}
\label{fig:budget}
\end{figure}

\begin{figure}[htbp]
\centering
\includegraphics[width=\columnwidth]{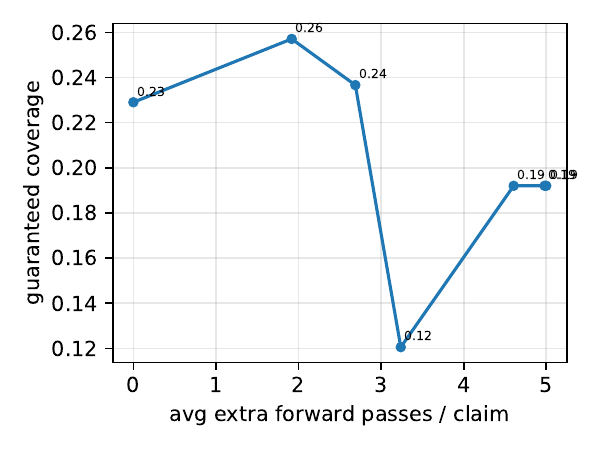}
\caption{Gating (three-way decision), exact guarantee, $\alpha{=}0.10$. Guaranteed coverage vs.\ average \emph{extra} forward passes per claim as the borderline band widens ($0$ = never acquire, $5$ = always). Targeted acquisition ($\approx2$ passes) beats full acquisition ($5$ passes): under the exact certificate, spending budget everywhere hurts.}
\label{fig:gate}
\end{figure}

\subsection{Guaranteed Guided Acquisition, and Is It Compute or Structure?}
\label{sec:tta}
The guided coverage of Table~\ref{tab:guided} is measured under the practical variant, so we re-run CLIP-guided acquisition through the \emph{exact} certificate of Table~\ref{tab:acq} (Bonferroni fixed-grid, $3{,}000$ one-per-image claims, $300$ splits, test-set exceedance $0$ throughout). Guaranteed coverage is $.06/.23/.41$ at $\alpha{=}0.05/0.10/0.20$, versus No-Acq's $.14/.23/.28$: full-budget guided certification beats abstention at $\alpha{=}0.20$ ($.41$ vs.\ $.28$) but not at tight $\alpha$, where acquiring on every claim inflates the score-max for false claims too---exactly what the \emph{gating} of Table~\ref{tab:acq} repairs. The guided gain is thus a genuine guarantee at moderate risk, not an artifact of the anti-conservative scan.

A natural worry is that any gain is merely the extra compute of $B$ additional forward passes. We separate compute from structure with a \emph{compute-matched} baseline on the subset for which \emph{both} uniform-grid and CLIP-guided crops were computed ($360$ one-per-image claims; Table~\ref{tab:tta}): the same $B{=}5$ uniform-grid crops, but aggregated the standard test-time-augmentation way---mean (or log-mean-exp) over crops---instead of our structured max-over-crops acquisition score. Because this is a smaller, harder subset than the $3{,}000$-claim set above, its absolute certified coverages are lower (No-Acq $.05$ at $\alpha{=}0.20$, vs.\ $.28$ on the full set); only the \emph{within-table} comparison at fixed data and compute is meaningful. At equal compute, TTA averaging \emph{does not help and can hurt}: pooling crops that usually miss the queried object dilutes the evidence, leaving guaranteed coverage at or below No-Acq. Only the structured max, which reads a crop as evidence \emph{that the object is present somewhere}, lifts detection ($0.807\!\to\!0.842$ AUROC), and CLIP targeting of the same budget lifts it further ($\to0.852$) and roughly doubles guaranteed coverage. The gain is in \emph{how} the acquired views are used, not in spending more passes. This compute-matched test, together with the ReCoVERR-style \textsc{Emp.-thr.}\ baseline of Table~\ref{tab:acq} (same crops, empirical risk threshold in place of our certificate), isolates our two claims---the value of the crop \emph{policy} and of finite-sample \emph{calibration}---at matched data and budget. A full reimplementation of accuracy-oriented active-perception systems (ViCrop, V*, ReCoVERR) under our certificate is orthogonal and left to future work; here the controlled comparison is against their operative ingredient rather than their engineering.

\begin{table}[htbp]
\centering
\setlength{\tabcolsep}{3pt}
{\small\begin{tabular}{lcc}
\toprule
score (same $B{=}5$ crops) & AUROC & cov @ $\alpha{=}0.20$ \\
\midrule
No-Acq (no crops) & $.807$ & $.05$ \\
TTA mean (equal compute) & $.817$ & $.01$ \\
TTA log-mean-exp (equal compute) & $.821$ & $.01$ \\
structured max (uniform grid) & $.842$ & $.02$ \\
\textbf{structured max (CLIP-guided)} & $\mathbf{.852}$ & $\mathbf{.09}$ \\
\bottomrule
\end{tabular}}
\caption{Compute-matched acquisition (LLaVA existence, exact certificate, $360$-claim subset with both crop types; coverage guaranteed at $\alpha{=}0.20$---absolute values are lower than the $3{,}000$-claim set and only the within-table comparison is meaningful). All acquiring rows use the \emph{same} five crops---only the aggregation differs (last row also CLIP-targets). Naive TTA averaging matches no-acquisition; the structured max and CLIP guidance carry the gain, so the improvement is structural, not compute.}
\label{tab:tta}
\end{table}

\begin{figure}[htbp]
\centering
\includegraphics[width=\columnwidth]{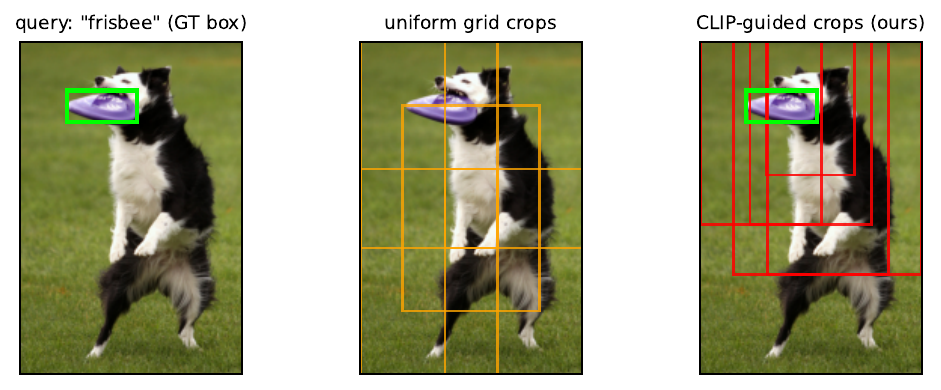}
\caption{Uniform grid crops (orange) spread blindly; CLIP-guided crops (red) concentrate on the queried small object (green ground-truth box), which the grid usually misses. No ground-truth location is used at inference---CLIP selects crops by similarity to ``a photo of a $\langle$object$\rangle$''.}
\label{fig:crop}
\end{figure}

\begin{figure}[htbp]
\centering
\includegraphics[width=\columnwidth]{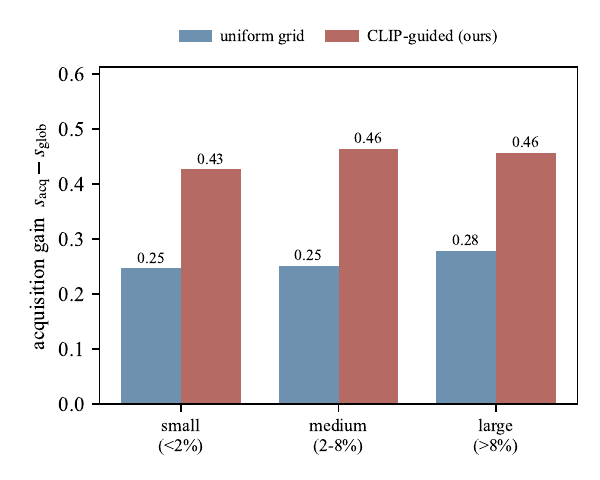}
\caption{Acquisition gain $s_{\text{acq}}{-}s_{\text{glob}}$ by object size. Uniform grid crops give size-independent gains; CLIP-guided crops nearly double the gain and, crucially, rescue \emph{small} objects the grid crops away.}
\label{fig:sizegain}
\end{figure}

\begin{figure*}[htbp]
\centering
\includegraphics[width=0.74\textwidth]{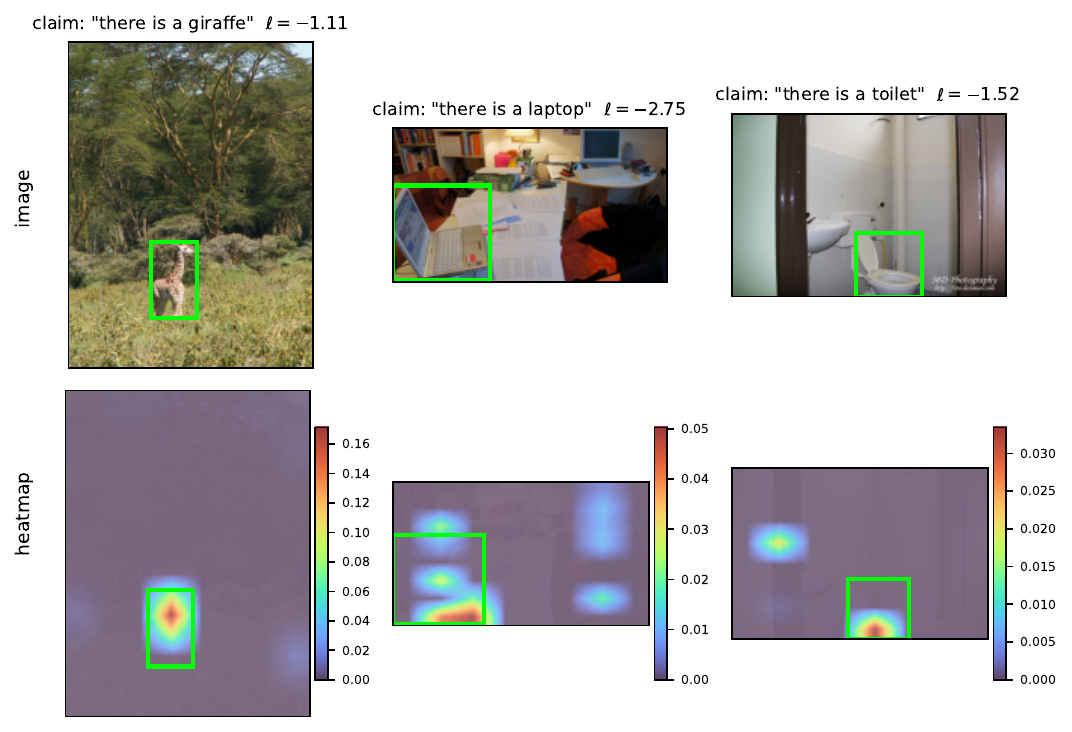}
\caption{Occlusion sensitivity for existence claims (LLaVA-1.5). \emph{Top:} image with the object's ground-truth box (green). \emph{Bottom:} we slide a black patch over an $8{\times}8$ grid and record the drop in the claim's mean log-likelihood (colorbar); brighter = occluding there hurts the claim more, i.e.\ the evidence is there. The high-sensitivity region concentrates on the queried object (green box), justifying region- and crop-based acquisition.}
\label{fig:heat}
\end{figure*}

\subsection{Active-Perception Policies Under Identical Calibration}
To isolate the acquisition \emph{policy} from the calibrator, we run several active-perception policies through the \emph{same} exact certificate on the same claims (Table~\ref{tab:activeperc}). A single targeted crop (a ViCrop-style saliency crop, $1$ extra pass) already beats naive multi-crop TTA averaging ($5$ passes)---\emph{targeting} matters more than \emph{count}---and our CLIP-guided max (BCEA) is strongest (AUROC $0.852$; guaranteed coverage $0.09$ at $\alpha{=}0.20$ vs.\ No-Acq's $0.05$). The gain is thus the policy, not merely crops plus a calibrator. A full reimplementation of the ViCrop/V*/ReCoVERR \emph{systems} is orthogonal engineering left to future work; here the controlled comparison is against their operative ingredient---a crop/search policy---under our guarantee.

\begin{table}[htbp]
\centering
\setlength{\tabcolsep}{4pt}
{\small\begin{tabular}{lccc}
\toprule
acquisition policy & +psg & AUROC & cov@$.20$ \\
\midrule
No-Acq & $0$ & $.807$ & $.05$ \\
uniform-grid TTA-mean & $5$ & $.817$ & $.01$ \\
uniform-grid max & $5$ & $.842$ & $.02$ \\
single guided crop (ViCrop-style) & $1$ & $.825$ & $.04$ \\
\textbf{CLIP-guided multi-crop (BCEA)} & $5$ & $\mathbf{.852}$ & $\mathbf{.09}$ \\
\bottomrule
\end{tabular}}
\caption{Active-perception policies under one \emph{identical} exact certificate (LLaVA existence, $360$ one-per-image claims, $\delta{=}0.1$; +psg = extra forward passes). Only the acquisition policy differs across rows; targeting (guided crops) beats brute count (grid TTA), and multi-crop guided is best.}
\label{tab:activeperc}
\end{table}

\subsection{Stricter Confidence Levels ($\delta$)}
The main text fixes the calibration confidence at $\delta{=}0.1$ ($90\%$). A higher-stakes deployment can demand more, and the exact certificate simply holds at stricter $\delta$: the empirical test-set exceedance stays $0$ at $\delta{=}0.05$ and $\delta{=}0.01$, while guaranteed coverage degrades only gently (Table~\ref{tab:deltasweep}; at $\alpha{=}0.10$, $0.23\!\to\!0.21\!\to\!0.19$ as $\delta$ tightens $0.10\!\to\!0.05\!\to\!0.01$; unchanged at $\alpha{=}0.20$). The choice $\delta{=}0.1$ is thus a convenience, not a limitation; safety-critical use can tighten it at a modest coverage cost.

\begin{table}[htbp]
\centering
\setlength{\tabcolsep}{6pt}
{\small\begin{tabular}{lccc}
\toprule
$\delta$ (test-set exceed.\ $=0$) & cov@$.05$ & cov@$.10$ & cov@$.20$ \\
\midrule
$0.10$ & $.06$ & $.23$ & $.41$ \\
$0.05$ & $.04$ & $.21$ & $.41$ \\
$0.01$ & $.01$ & $.19$ & $.41$ \\
\bottomrule
\end{tabular}}
\caption{Guaranteed coverage under stricter calibration confidence $\delta$ (BCEA guided, exact certificate, $3{,}000$ one-per-image claims, $300$ splits). The certificate holds at every $\delta$ (empirical test-set exceedance $0$); coverage costs little to tighten confidence to $99\%$.}
\label{tab:deltasweep}
\end{table}

\subsection{Generality Across Models}

We repeat the core experiments on Qwen3-VL-7B \citep{bai2025qwen3} (Table~\ref{tab:qwen}). The findings transfer. On existence claims, acquisition lifts the grounding AUROC from $0.70$ to $0.77$, and BCEA improves coverage over abstention at every level---most strikingly at $\alpha{=}0.05$, where abstention can assert almost nothing ($1\%$) while BCEA asserts $15\%$ at the same practical target risk (practical variant). On spatial relations, the global ungrounding score is again at chance ($0.50$ AUROC), confirming that the failure of global scores on relational claims is not specific to one model. Acquisition-and-recalibration is thus a model-agnostic recipe; the absolute grounding quality and the strength of a given intervention vary by backbone.

\begin{table}[htbp]
\centering
{\small\begin{tabular}{lccc}
\toprule
\multicolumn{4}{l}{\emph{Existence: practical coverage at target risk $\alpha$}}\\
$\alpha$ & \textsc{No-Acq} & \textsc{BCEA} & risk \\
\midrule
$0.05$ & $0.01$ & $\mathbf{0.15}$ & $0.05$ \\
$0.10$ & $0.13$ & $\mathbf{0.22}$ & $0.10$ \\
$0.20$ & $0.22$ & $\mathbf{0.30}$ & $0.21$ \\
\midrule
\multicolumn{4}{l}{AUROC: $s_{\text{glob}}{=}0.70$, $s_{\text{acq}}{=}0.77$; spatial $s_{\text{glob}}{=}0.50$}\\
\bottomrule
\end{tabular}}
\caption{BCEA on Qwen3-VL-7B. The acquisition-and-recalibration recipe transfers: BCEA beats guaranteed abstention at every risk level, and global scoring is again at chance for relations.}
\label{tab:qwen}
\end{table}

\subsection{Comparison to \textsc{ConfLVLM} Scoring}

Our No-Acq baseline is a conformal grounding filter of our own; to compare against the published alternative we re-implement the scoring functions of \textsc{ConfLVLM} \citep{li2025towards}---claim log-probability given the image, given \emph{text only} (their language-prior reference), their ratio, and CLIP image--text similarity---and certify each with the same exact procedure on the same claims (Table~\ref{tab:conflvlm}). Our evidence-sufficiency score, which references a \emph{blank image} rather than text-only, is the stronger signal ($0.807$ vs.\ $0.769$ AUROC), and acquisition widens the gap to $0.852$, roughly doubling guaranteed coverage at $\alpha{=}0.20$ ($0.09$ vs.\ $0.05$). The comparison isolates the scoring function; it is not a reimplementation of the full \textsc{ConfLVLM} system.

\begin{table}[htbp]
\centering
\setlength{\tabcolsep}{4pt}
{\small\begin{tabular}{lcc}
\toprule
score & AUROC & cov/viol$_{\alpha=.2}$ \\
\midrule
\textsc{ConfLVLM} ratio (img$-$text) & $.769$ & $.02$/$.01$ \\
\textsc{ConfLVLM} CLIP sim. & $.785$ & $.00$/$.00$ \\
\textsc{ConfLVLM} $\log p(c\!\mid\!\text{img})$ & $.795$ & $.06$/$.00$ \\
ours $s_{\text{glob}}$ (img$-$blank) & $.807$ & $.05$/$.00$ \\
\textbf{ours $s_{\text{acq}}$ (BCEA)} & $\mathbf{.852}$ & $\mathbf{.09}$/$.00$ \\
\bottomrule
\end{tabular}}
\caption{Scoring-function head-to-head under the same exact certification ($360$ claims, one per image). Referencing a blank image beats the text-only reference; acquisition widens the gap.}
\label{tab:conflvlm}
\end{table}

\subsection{Additional Visualizations}

\begin{table}[htbp]
\centering
\setlength{\tabcolsep}{4pt}
{\small\begin{tabular}{lc}
\toprule
Score & AUROC (correct vs.\ swapped) \\
\midrule
Raw claim likelihood & $0.567$ \\
$s_{\text{glob}}$ (global ungrounding) & $0.574$ \\
$s_{\text{flip}}$ (flip intervention, ours) & $\mathbf{0.765}$ \\
\bottomrule
\end{tabular}}
\caption{Spatial left/right relations. The global score used by prior conformal filters is near chance; a cheap structured intervention recovers the signal.}
\label{tab:spatial}
\end{table}

\begin{table*}[htbp]
\centering
\setlength{\tabcolsep}{4pt}
{\small\begin{tabular}{llcccc}
\toprule
Claim & base score & base decision & acquired view & BCEA score & BCEA decision \\
\midrule
``there is a kite''                 & $0.08$ ($<\tau_{\rm base}$) & \textcolor{red!70!black}{Abstain} & CLIP crop zooms onto kite       & $1.30$ & \textcolor{green!45!black}{\textbf{Assert}} \\
``there is a sink''                 & $0.09$ ($<\tau_{\rm base}$) & \textcolor{red!70!black}{Abstain} & CLIP crop zooms onto sink       & $1.13$ & \textcolor{green!45!black}{\textbf{Assert}} \\
``there is a horse''                & $0.12$ ($<\tau_{\rm base}$) & \textcolor{red!70!black}{Abstain} & CLIP crop zooms onto horse      & $1.29$ & \textcolor{green!45!black}{\textbf{Assert}} \\
``ball is left of person'' (rel.)   & $\approx$ chance   & \textcolor{red!70!black}{Abstain} & horizontal flip flips the truth & $s_{\text{flip}}{=}{+}0.14$ & \textcolor{green!45!black}{\textbf{Assert}} \\
\bottomrule
\end{tabular}}
\caption{Qualitative comparison on truly-supported claims (LLaVA-1.5, $\alpha{=}0.10$; $\tau_{\rm base}{=}0.16$, recalibrated $\tau_{\rm BCEA}{=}0.49$). The guaranteed baseline abstains because the full-image evidence score is below threshold (or, for relations, uninformative); BCEA acquires a targeted view that lifts the score above the recalibrated $\tau_{\rm BCEA}$ and asserts---each row a question answered with evidence rather than declined.}
\label{tab:rescue}
\end{table*}

\begin{figure}[htbp]
\centering
\includegraphics[width=0.82\columnwidth]{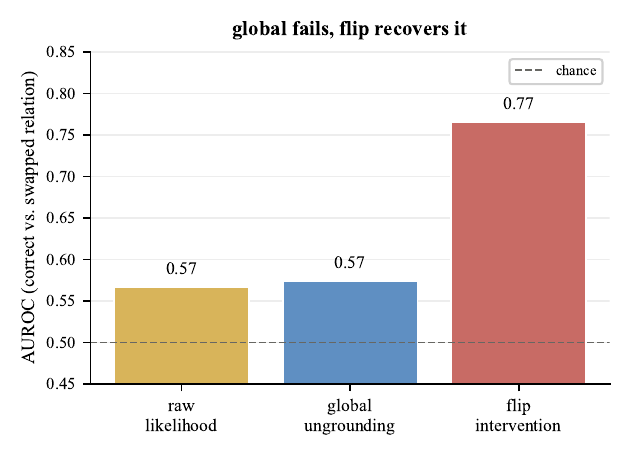}
\caption{Spatial left/right relations ($996$ balanced claims). The global ungrounding score is near chance because both objects are present; a single horizontal-flip intervention recovers the relational signal.}
\label{fig:spatial}
\end{figure}

\section{Extended Discussion and Limitations}

Our study uses four backbones and two working claim families (existence, left/right); a principled per-claim-type intervention library is the natural next step. The method has a clear scope boundary: evidence sufficiency certifies only what the backbone can perceive. Counting, relative size, above/below, color, and color--object \emph{binding} are near chance ($s_{\text{glob}}$ AUROC $0.50$--$0.58$) and no intervention we tried recovers signal, so coverage there collapses to zero. This is faithful behavior---where a $7$--$8$B VLM cannot perceive a property, our score reports \emph{ungrounded} rather than fabricating confidence, and the filter abstains. BCEA adds signal only where the backbone can localize and a truth-flipping, in-distribution intervention exists; per-claim budget \emph{allocation}, by contrast, is coverage-neutral here (a clarifying negative). The framework is otherwise general: any black-box LVLM, any claim-type intervention, any exchangeable calibration set.

\paragraph{From probed to free-form claims.} Most experiments verify \emph{probed} claims (clean labels, controlled difficulty). The free-form pipeline (Table~\ref{tab:freeform}) closes much of the gap---BCEA certifies claims LLaVA volunteers, $20.6\%$ hallucinated---but our parser extracts \emph{object mentions} only, so attributes, relations and counts in a caption stay unverified and parser noise is uncharged. A full atomic-fact decomposer is future work.

\bibliography{aaai2027}
\end{document}